\documentclass[10pt,technote]{IEEEtran}
\ifCLASSOPTIONcompsoc
  \usepackage[nocompress]{cite}
\else
  \usepackage{cite}
\fi

\usepackage{amsmath}
\usepackage{amssymb}
\usepackage{amsthm}
\newtheorem{mythm}{Theorem}
\newtheorem{mycorr}{Corollary}
\usepackage{pdfpages}
\usepackage{todonotes}

\newcommand\tb[1]{\todo[inline]{\emph{@TB: #1}}}

\usepackage{dcolumn}
\usepackage{algpseudocode}
\usepackage{algorithmicx}
\usepackage{algorithm}

\usepackage{graphicx}

\usepackage{hyperref}
\def\PSIp{\hbox{$\Psi_{C}$}}
\def\PSIn{\hbox{$\Psi_{M}$}}

\def\tailsize{\tau}
\long\def\comment#1{}
\def\PSI{\hbox{$\Psi$}}
\newcommand{\argmax}{\arg\!\max}

\makeatletter
\renewcommand\subsection{\@startsection
  {subsection}{2}{0mm}
  {-1ex}
  {0.5ex}
  {\normalfont\normalsize\itshape}}
\makeatother
\begin{document}
\title{The Extreme Value Machine}
\author{Ethan~M.~Rudd, Lalit~P.~Jain,\\
        Walter~J.~Scheirer,~\IEEEmembership{Senior Member,~IEEE,}\\
        and~Terrance~E.~Boult,~\IEEEmembership{IEEE Fellow}
\IEEEcompsocitemizethanks{\IEEEcompsocthanksitem E. M. Rudd, L. P. Jain, and T. E. Boult are with the Department of Computer Science, University of Colorado Colorado Springs, Colorado Springs, CO, 80918. W. J. Scheirer is with the Department of Computer Science and Engineering, University of Notre Dame, Notre Dame, IN, 46556.\protect\\
Corresponding Author's E-mail: erudd@vast.uccs.edu}
\thanks{}}

\markboth{IEEE TRANSACTIONS ON PATTERN ANALYSIS AND MACHINE INTELLIGENCE,~Vol.??, No.~??, March~2017}%
{Shell \MakeLowercase{\textit{et al.}}: Bare Demo of IEEEtran.cls for Computer Society Journals}

\IEEEtitleabstractindextext{%

\begin{abstract}
It is often desirable to be able to recognize when inputs to a recognition function learned in a supervised manner correspond to classes unseen at training time. With this ability, new class labels could be assigned to these inputs by a human operator, allowing them to be incorporated into the recognition function --- ideally under an efficient incremental update mechanism. While good algorithms that assume inputs from a fixed set of classes exist, \textit{e.g.}, artificial neural networks and kernel machines, it is not immediately obvious how to extend them to perform incremental learning in the presence of unknown query classes. Existing algorithms take little to no distributional information into account when learning recognition functions and lack a strong theoretical foundation. We address this gap by formulating a novel, theoretically sound classifier --- the Extreme Value Machine (EVM). The EVM has a well-grounded interpretation derived from statistical Extreme Value Theory (EVT), and is the first classifier to be able to perform nonlinear kernel-free variable bandwidth incremental learning. Compared to other classifiers in the same deep network derived feature space, the EVM is accurate and efficient on an established benchmark partition of the ImageNet dataset.
\end{abstract}


\begin{IEEEkeywords}
Machine Learning, Supervised Classification, Open Set Recognition, Open World Recognition, Statistical Extreme Value Theory
\vspace{2em}
\end{IEEEkeywords}}

\maketitle
\IEEEdisplaynontitleabstractindextext

\IEEEpeerreviewmaketitle
\IEEEraisesectionheading{\section{Introduction}\label{sec:introduction}}

\IEEEPARstart{R}{ecognition} problems which evolve over time require classifiers that can incorporate novel classes of data. What are the ways to approach this problem? One is to periodically retrain classifiers. However, in situations that are time or resource constrained, periodic retraining is impractical. Another possibility is an online classifier that incorporates an efficient incremental update mechanism. 
While methods have been proposed to solve the incremental learning problem, they are computationally expensive~\cite{crammer2006online,yeh2008dynamic,laskov2006incremental,li2010optimol}, or provide little to no characterization of the statistical distribution of the data~\cite{bendale2014towards,kapoor2012memory,mensink2012metric,ristin2014incremental}. The former trait is problematic because it is contrary to the fundamental motivation for using incremental learning --- that of an efficient update system ---  while the latter trait places limitations on the quality of inference.

There is also a more fundamental problem in current incremental learning strategies. When the recognition system encounters a novel class, that class should be incorporated into the learning process at subsequent increments. But in order to do so, the recognition system needs to identify novel classes in the first place. For this type of \textit{open set} problem in which unknown classes appear at query time, we cannot rely on a \textit{closed set} classifier, even if it supports incremental learning, because it implicitly assumes that all query data is well represented by the training set. 

\begin{figure}[t]\begin{center} \includegraphics[width=0.85\linewidth]{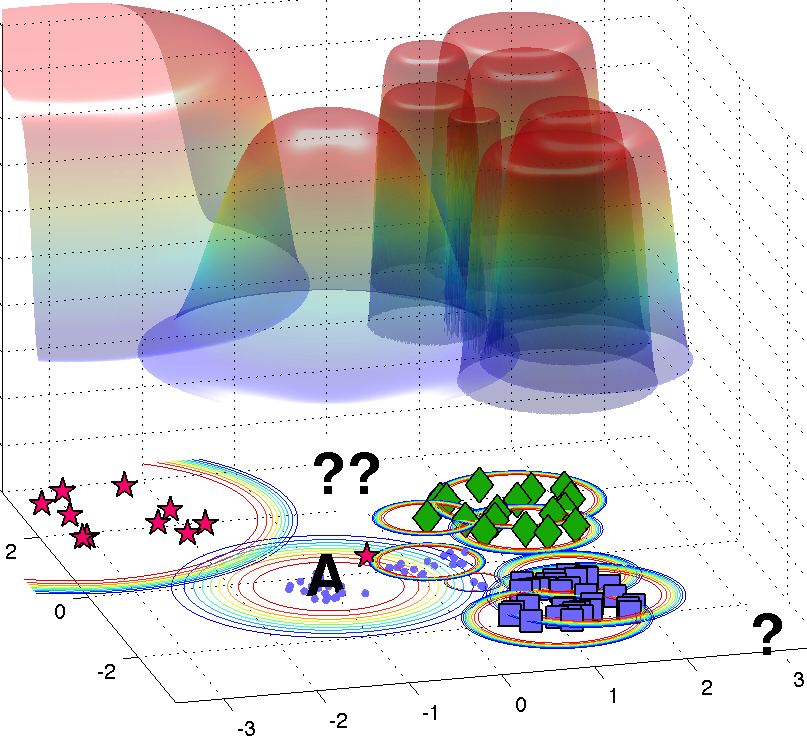}\vspace*{-2.5ex}\end{center}
\caption{\small 
A solution from the proposed EVM algorithm trained on four classes: dots, diamonds, squares, and stars.
The colors in the isocontour rings show a $\PSI$-model (probability of sample inclusion) for each \textit{extreme vector} (EV) chosen by the algorithm, with red near 1 and blue near .005. 
Via kernel-free non-linear modeling, the EVM supports open set recognition and can reject the three ``?" inputs that lie beyond the support of the training set as ``unknown.'' 
Each $\PSI$-model has its own independent shape and scale parameters learnt from the data, and supports a soft-margin. For example, the \PSI-model for the blue dots corresponding to extreme vector A has a more gradual fall off, due to the effect of the outlier star during training.
}
\label{fig:teaser}
\vspace{-2em}
\end{figure}

Closed set classifiers have been developed for approximating the Bayesian optimal posterior probability, $P(C_l |x';{\cal C}_1, {\cal C}_2, \hdots, {\cal C}_M), l \in \{1,\hdots, M\}$, for a fixed set of classes, where $x'$ is an input sample, $l$ is the index of class ${\cal C}_l$ (a particular known class), and $M$ is the number of known classes. When $\Omega$ unknown classes appear at query time, however, the Baysian optimal posterior becomes $P( C_{\tilde{l}} |x';{\cal C}_1, {\cal C}_2, \hdots, {\cal C}_M, U_{M+1}, \hdots, U_{M+\Omega}), \tilde{l} \in \{1, \hdots, M +\Omega\}$, a distribution that we cannot model because classes $U_{M+1}$ through $U_\Omega$ are unknown. Making closed set assumptions in training leads to regions of unbounded support for an open set problem because a sample $x'$ from an unknown class $U_{\tilde{l}}$ will be misclassified as a known class ${\cal C}_{l}$. For classifiers that assess confidence in terms of signed distance from a decision boundary, or some calibration thereof, this misclassification will occur with high confidence if $x'$ is far from any known data --- a result that is very misleading. Scheirer et al.~\cite{Walter_openset} termed this problem \textit{open space risk}.

More formally, let $f$ be a measurable recognition function for known class ${\cal C}$, ${\cal
  O}$ be the  open space, and $S_o$ be a ball of radius $r_o$ that includes all of the known
positive training examples $x \in {\cal C}$ as well as the open space $\cal
O$. Open space risk $R_{\cal O}(f)$ for class ${\cal C}$ can be defined as $R_{{\cal O}}(f) =  \frac{\int_{{\cal O}} f_{\cal C}(x) dx}{\int_{S_o} f_{\cal C}(x) dx} $, where open space risk is considered to be the relative measure of positively
labeled open space compared to the overall measure of positively labeled space.   In this probabilistic formulation, the objective of \textit{open set recognition} is to exercise a rejection option~\cite{bishop2006pattern} for queries that lie beyond the reasonable support of known data, thus mitigating this risk.

\comment{
Given a fixed set of classes, most closed set classifiers work well for approximating the Bayesian optimal posterior probability $P(C_l |x';{\cal C}_1, {\cal C}_2, \hdots, {\cal C}_M), l \in \{1,\hdots, M\}$, where $x'$ is an input sample, $l$ is the index of class ${\cal C}_l$ (a particular known class), and $M$ is the number of known classes. 
However, the Bayesian optimal posterior probability for an open set problem with $\Omega$ unknown classes becomes $P( C_{\tilde{l}} |x';{\cal C}_1, {\cal C}_2, \hdots, {\cal C}_M, U_{M+1}, \hdots, U_{M+\Omega}), \tilde{l} \in \{1, \hdots, M +\Omega\}$.
We cannot model this distribution because classes $U_{M+1}$ through $U_\Omega$ are unknown. A closed set assumption in classifier training leads to regions of unbounded support when an attempt is made to apply the classifier to an open set problem, \textit{i.e.}, a sample $x'$ from an unknown class $U_{\tilde{l}}$ will be classified as belonging to a known class $C_{l}$. Scheirer et al. termed this phenomenon \textit{open space risk}~\cite{Walter_openset}.
}

Open set recognition~\cite{Walter_openset,Lalit,walter2014}, and more generally novelty/outlier detection~\cite{markou2003novelty,hodge2004survey} are well established areas in their own right, but much less research has been conducted on how to treat unknown samples in an incremental context, which is the focus of this work. When an open set recognition system labels a sample as unknown, it suggests that the model was not trained with data from the class corresponding to that sample. In response, the classifier's decision boundaries should be updated so that the system can incorporate this new class information for future decision making. But there is a caveat: full retraining is not always feasible, depending on timing constraints and the availability of computational resources.

Recent work~\cite{bendale2014towards} extended the open set recognition problem to include the incremental learning of new classes in a regime dubbed \textit{open world recognition}, which is the problem we are most concerned with in this paper. An effective open-world recognition system must perform four tasks: detecting unknowns, choosing which points to label for addition to the model, labelling the points, and updating the model. 
An algorithm, nearest non-outlier (NNO), was proposed as a demonstration of these elements --- the first of its kind. Unfortunately, NNO lacks strong theoretical grounding. The algorithm uses thresholded distances from the nearest class mean as its decision function, and otherwise ignores distributional information. 
Weak classifiers are a persistent problem for this task: it is not immediately obvious how one might extend  class boundary models from classical machine learning theory (\textit{e.g.}, neural networks and kernel machines) to incorporate both incremental learning and open set constraints. A new formulation is required.

 In this article we address the construction of a compact representation of open world decision boundaries based on the distribution of the training data. Obtaining this representation is difficult because  training points that do not contribute to a decision boundary at one point in time may be extremely relevant in defining a decision boundary later on, and retraining on all points is infeasible at large scales. Moreover, by the definition of the open world problem, the hypothesis space will be under-sampled, so in many cases linearity of the decision boundaries cannot be guaranteed and the data bandwidth is unknown.
So how does one obtain a compact statistical model without discarding potentially relevant points --- especially in regions where the data bandwidth is unknown? To this end, we introduce the Extreme Value Machine (EVM), a model which we derive from statistical extreme value theory (EVT). 

EVT dictates the functional form for the radial probability of inclusion of a point with respect to the class of another. By selecting the points and distributions that best summarize each class, \textit{i.e.}, are least redundant with respect to one another, we arrive at a compact probabilistic representation of each class's decision boundary, characterized in terms of its \textit{extreme vectors} (EV), which provides an abating bound on open space risk. This is depicted in schematic form in Fig. \ref{fig:teaser}. When new data arrives, these EVs can be efficiently updated. The EVM is a scalable nonlinear classifier, with radial inclusion functions that are in some respects similar to RBF kernels, but unlike RBF kernels assume variable bandwidths and skew that are data derived and grounded in EVT.

\section{Related Work}\label{sec:relatedwork}

With respect to classifiers that mitigate open space risk at classification time, the 1-vs-Set machine~\cite{Walter_openset} approaches the problem of open set recognition by replacing the half-space of a binary linear classifier by bounding the positive data with two hyperplanes. An algorithm similar to the 1-vs-Set machine was described by Cevikalp and Triggs~\cite{cevikalp_2012} for object detection, where a binary classifier with a slab is combined with a nonlinear SVDD classifier for just the positive class. In later work, Scheirer et al. introduced the W-SVM for multi-class open set recognition problems using nonlinear kernels, with provable guarantees of open space risk reduction~\cite{walter2014}. These nonlinear models were more accurate, but also more costly to compute and store. For the more expansive problem of open world recognition, Bendale and Boult modified the Nearest Class Mean~\cite{mensink2013distance} algorithm by limiting open space risk for model combinations and transformed spaces, resulting in the NNO algorithm introduced in Sec.~\ref{sec:introduction}, which we will use as a baseline for comparison in Sec.~\ref{sec:experiments}. 

Other approaches exist for related problems involving unknown class data such as multi-class novelty detection~\cite{noveltydetection}, domain adaptation~\cite{domainadaptation}, and zero-shot classification~\cite{zero-shot}. However, these problems need not be addressed by a classifier that is incrementally updated over time with class-specific feature data. More related is the problem life-long learning, where a classifier receives tasks and is able to adapt its model in order to perform well on new task instances. Pentina and Ben-David~\cite{pentina2015multi} lay out a cogent theoretical framework for SVM-based life-long learning tasks, but leave the door open to specific implementations that embody it. Along these lines, Royer and Lampert~\cite{royer2015classifier} describe a method for classifier adaptation that is effective when inherent dependencies are present in the test data. This works for fine-grained recognition scenarios, but does not address unknown classes that are well separated in visual appearance from the known and other unknown classes.  The problem most related to our work is rare class discovery, for which Haines and Xiang have proposed an active learning method that jointly addresses the tasks of class discovery and classification~\cite{haines2014active}. We consider their classification algorithm in Sec.~\ref{sec:experiments}, even though we do not make distinctions between common and rare unknown classes.

There is growing interest in statistical extreme value theory for visual recognition. With the observation that the tails of any distance or similarity score distribution must always follow an EVT distribution~\cite{walter2011}, highly accurate probabilistic calibration models became possible, leading to strong empirical results for multi-biometric fusion~\cite{Walter_Wscore}, describable visual attributes~\cite{walter2012}, and visual inspection tasks~\cite{gibert2015sequential}. EVT models have also been applied to feature point matching, where the Rayleigh distribution was used for efficient guided sampling for homography estimation~\cite{victor2013a}, and the notion of extreme value sample consensus was used in conjunction with RANSAC for similar means~\cite{victor2013}. Work in machine learning has shown that EVT is a suitable model for open set recognition problems, where one-~\cite{Lalit} and two-sided calibration models~\cite{walter2014,zhang2016sparse} of decision boundaries lead to better generalization. 
However, these are \textit{post hoc} approaches that do not apply EVT at training time.

\section{Theoretical Foundation}
\label{sec:margin_distributions}

As discussed in Sec.~\ref{sec:introduction} and as illustrated in Fig.~\ref{fig:teaser}, each class in the EVM's training set is represented by a set of extreme vectors, where each vector is associated with a radial inclusion function modeling the \textit{Probability of Sample Inclusion} (PSI or \PSI). 
Here we derive the functional form for \PSI\ from EVT; this functional form is not just a mathematically convenient choice --- it is statistically guaranteed to be the limiting distribution of relative proximity between data points under the minor assumptions of continuity and smoothness.

The EVM formulation developed herein stems from the concept of \textit{margin distributions}. This idea is not new;  various definitions and uses of margin distributions have been explored\cite{shawe2000margin,garg2003margin,reyzin2006boosting,aiolli2008kernelo,pelckmans2008risk}, involving techniques such as maximizing the mean or median margin, taking a weighted combination margin, or optimizing the margin mean and variance. Leveraging the margin distribution itself can provide better error bounds than
those offered by a soft-margin SVM classifier, which in some cases translates
into reduced experimental error. 
We model \PSI\ in terms of the distribution of sample half-distances relative to a reference point, extending margin distribution theory from a per-class formulation~\cite{shawe2000margin,garg2003margin,reyzin2006boosting,aiolli2008kernelo,pelckmans2008risk} to a sample-wise formulation. 
The model is fit on the distribution of margins --- half distances to the nearest negative samples --- for each positive reference point. 
From this distribution, we derive a radial inclusion function which carves out a posterior probability of association with respect to the reference point. This radial inclusion function falls toward zero at the margin.  

\comment{
But the model comes with a caveat: when dimensionality is large and when classes highly overlap, modeling using \PSIn\ alone can lead to overspecialized ``spiky'' regions of support. 
While one could improve coverage by increasing the number of negative samples in fitting, this greatly increases computational cost.
}

\comment{
We therefore introduce a \emph{coverage probability of sample inclusion} (\PSIp) model, trained by fitting on positive data.
By introducing \PSIp, we seek to ensure that the local model covers (\textit{i.e.}, indicates high-probability for) other positive examples. 
While \PSIn\ is a rejection model, \PSIp\ is an acceptance model derived from a distribution fit on the positive tail of the distances to the farthest $\tailsize$ positive points that lie closer to the reference sample than the nearest negative.
Training \PSIp\ requires only one negative point for discrimination. 
Note that to increase coverage, \PSIp\ uses complete distances instead of margins for the selected positive points. 
Empirically, \PSIp\ accommodates areas with class overlap or high-dimensional features where margins for \PSIn\ may over-specialize.
}

\subsection{Probability of Sample Inclusion}


To formalize the \PSI-model, let $x \in {\cal X}$ be training samples in a feature space $\cal X$. Let $y_i \in {\cal C} \in \mathbb{N}$ be the class label for $x_i\in {\cal X}$. 
Consider, for now, only a single positive instance ${x}_i$ for some class with label $y_i$. 
Given ${x}_i$, the maximum margin distance would be given by half the distance to the closest training sample from a different class. 
However, the closest point is just one sample and we should consider the potential maximum margins under different samplings. 
We define margin distribution as the distribution of the margin distances of the observed data. 
Thus, given ${x}_i$ and $x_j$, where $\forall j$, $ y_j \ne y_i$, consider the margin distance to the decision boundary that would be estimated for the pair $({x}_i,x_j)$ if $x_j$ were the closest instance. 
The margin estimates are thus $m_{ij} = \|{x}_i - x_j\|/2$ for the $\tailsize$ closest points. 
Considering these $\tailsize$ nearest points to the margin, our question then becomes: what is the distributional form of the margin distances? 

To estimate this distribution, we turn to the Fisher-Tippett Theorem~\cite{fisher1928limiting} also known as the Extreme Value Theorem\footnote{There are other types of extreme value theorems, \textit{e.g.}, the second extreme value theorem, the Pickands-Balkema-de Haan Theorem~\cite{Pickands_1975}, addresses probabilities conditioned on the process exceeding a sufficiently high threshold.}. 
Just as the Central Limit Theorem dictates that the random variables generated from certain stochastic processes follow Gaussian distributions, EVT dictates that given a well-behaved overall distribution of values, \textit{e.g.}, a distribution that is continuous and has an inverse, the distribution of the maximum or minimum values can assume only limited forms. 
To find the appropriate form, let us first recall:
\begin{mythm}[Fisher-Tippett Theorem~\cite{Kotz_2001}]\label{EVTtheorem}
Let $(v_1, v_2, \ldots)$ be a sequence of i.i.d samples. Let $\zeta_n = \mathrm{max}\{v_1, \ldots, v_n\}.$ If a sequence of pairs of real numbers $(a_n, b_n)$ exists such that each $a_n > 0$ and $
\lim_{z\to\infty} P \big( \frac{\zeta_n - b_n}{a_n} \le z \big) = F(z)
$ then if $F$ is a non-degenerate distribution function, it belongs to the Gumbel, the Fr\'{e}chet or the Reversed Weibull family.
\end{mythm}
In other words for any  sequence $(a_n, b_n)$ of shifts and normalizations of the samples such that the probability of the maximum value converges, it converges to one of three distributions\footnote{A more thorough introductory overview of EVT can be found in \cite{coles2001introduction}.}. 
From Theorem~\ref{EVTtheorem}, we can derive the following:
\comment{This extreme value theorem is widely used in many fields \cite{Kotz_2001}, such as manufacturing (e.g., estimating time to failure), natural sciences (e.g., estimating 100 or 500 year flood levels), and finance (e.g., portfolio risks). EVT has recently been (re)introduced as applying to recognition,  machine learning, and computer vision\cite{walter2011,walter2014,carpentier2014extreme}. 
Here we  use it to formalize the per-sample input  margin distribution:}
\begin{mythm}[Margin Distribution Theorem]
\label{max_margin_distribution}
Assume we are given a positive sample $x_i$ and sufficiently many negative samples $x_{j}$ drawn from well-defined class distributions, yielding pairwise margin estimates $m_{ij}$. 
Assume a continuous non-degenerate margin distribution exists. 
Then the distribution for the minimal values of the margin distance for $x_i$ is given by a Weibull distribution. 
\end{mythm}

\begin{proof} Since Theorem~\ref{EVTtheorem} applies to maxima, we transform the variables via $z = -m_{ij}$ and consider the maximum set of values $-m_{ij}$. The assumption of sufficient samples and a well-defined set of margin distances converging to a non-degenerate margin implies that Theorem~\ref{EVTtheorem} applies.  Let $\phi$ be the associated distribution of the maxima of $\zeta_n$. Combining  Theorem~\ref{EVTtheorem} with knowledge that  the data are bounded ($-m_{ij}<0$) means that $\phi$ converges to a reversed Weibull, as it is the EVT distribution that is bounded from above~\cite{Kotz_2001}. Changing the variables back ($m_{ij}=-z$) means that the minimum distance to the boundary must be a Weibull distribution.\end{proof}

Theorem 2 holds for any point $x_i$, with each point estimating its own distribution of distance to the margin yielding:

\begin{mycorr}[\PSI\ Density Function]
\label{margin_distribution_kernel}
Given the conditions for the Margin Distribution Theorem, the probability that $x^\prime$ is included in the boundary estimated by $x_i$ is given by:
\begin{equation}
\PSI(x_i,x^\prime;\kappa_i, \lambda_i,) = \exp^{-\left(\frac{||x_i-x^\prime||}{\lambda_i}\right)^{\kappa_i}}
\label{eq:weibull_cdf}
\end{equation}
where $||x_i-x^\prime||$ is the distance of $x^\prime$ from sample $x_i$, and $\kappa_i$, $\lambda_i$ are Weibull shape and scale parameters respectively obtained from fitting to the smallest $m_{ij}$. 
 \end{mycorr}
\begin{proof} 
The Weibull cumulative distribution function (CDF) $F_{W}(||x'-x_i||;\kappa_i,\lambda_i)$ provides the probability that the margin is at or below a given value, but we seek the probability that $x^\prime$ does not exceed the margin yielding the inverse: $1 - F_{W}(||x'-x_i||;\kappa_i,\lambda_i)$ (cf. Eq.~\ref{eq:weibull_cdf}).\end{proof}

The \PSI-model defines a radial inclusion function that is an EVT rejection model where the probability of inclusion corresponds to the probability that the sample \textit{does not} lie well into or beyond the negative margin. While \PSI\ is designed to have zero probability around the margin, half-way to the negative data, the model still supports a soft margin because the EVT estimation uses $\tailsize$ points and hence may cover space with both positive and negative points. However, this does not force the model to include {\em any} positive training samples within its probability of inclusion. 

\subsection{Decision Function}

\comment{
We have two models for a given point: \PSIn\ and \PSIp, but how do we decide which one to use or how to fuse them? 
Cross-validation is nontrivial for two reasons: First, it introduces additional computation, thus hindering scalability. 
Second, it does not let us reason beyond the error of the validation set, which gives little indication that the model will generalize in the presence of unknown classes or when classes are incrementally added later on. Fortunately, because \PSIn\ and \PSIp\ have similar functional forms, we can easily fuse them. Since a reversed Weibull PDF is a reflection of a Weibull PDF at the same scale and shape, we can decide whether to use \PSIn\ or \PSIp\ via a comparison of the respective scale parameters.

We let the \textit{overall probability of sample inclusion} be:
\begin{equation}
 \hat\PSI(x_i,x'; \hat\kappa_i, \hat\lambda_i) =\begin{cases} \PSIn(x_i,x;\kappa_i, \lambda_i) & \hbox{if } \lambda^\prime_i < \lambda_i \\
 \PSIp(x_i,x';\kappa'_i, \lambda'_i) & \hbox{otherwise.} \\
 \end{cases}
\end{equation}
Taking the model with larger scale gives us a principled choice of \PSIp\ or \PSIn\ in the following respect: \PSIp\ only defines the distribution of points closer than the nearest negative, so we expect its scale to be small. If \PSIn\ predominantly falls in the space governed by this distribution, then \PSIn\ is spiky, suggesting poor coverage of positive data. In this case, we leave the classification decision to \PSIp. Otherwise, we trust the support of the learnt margin distribution.
}

The probability that a query point $x'$ is associated with class ${\cal C}_l$  is  $\hat P({\cal C}_l|x') = \argmax_{\{i: y_i = {\cal C}_l\}} \PSI\ (x_i,x'; \kappa_i, \lambda_i)$. 
Given $\hat P$, we compute the open set multi-class recognition result for $x^\prime$. Let threshold $\delta$ on the probability define the boundary between the set of known classes ${\cal C}$ and unsupported open space~\cite{Walter_openset} so that the final classification decision is given by:
\begin{equation}
y^* = \begin{cases}\argmax_{l \in \{1,...,M\}} \hat P({\cal C}_l|x^\prime) 
 & \hbox{if } \hat P({\cal C}_l|x^\prime)  \ge \delta \\
\hbox{``unknown"} &  \hbox{Otherwise.}  
\end{cases}
\label{eq:thresholdprob}
\end{equation}

A principled approach to selecting $\delta$ is to optimize the tradeoff between closed set accuracy and rejection of unknown classes via \textit{cross-class validation} \cite{bendale2014towards} on the training set, leaving out a subset of classes as ``unknown'' at each fold.
A slight generalization of the decision function in Eq.~\ref{eq:thresholdprob} is to average over the $k$-largest probabilities for each class.
For all experiments in Sec.~\ref{sec:experiments} we selected a $k$ value via a hyperparameter search on non-test data, choosing $k$ from $\{1,\hdots,10\}$.
In practice, we found that a choice of $k > 1$ yields only slight performance gains of 1-2\% in accuracy.  


\section{EVM Formulation}
\label{sec:evmformulation}

With the \PSI-models, we can develop an algorithm that is not only advantageous for open world recognition, but is also useful for limiting trained model size and obtaining favorable scaling characteristics. The pseudocode for this algorithm is provided in the supplemental material for this paper. Corresponding source code will be made available after publication.

\subsection{Model Reduction}
\label{sec:model_reduction}

Keeping all \PSI-models and associated data points results in larger models and longer classification times as dataset sizes increase, which is undesirable in both incremental and resource constrained scenarios. The success of sparse classification algorithms in other problem domains (\textit{e.g.}, SVM) suggests that we can strategically discard many redundant $\langle x_i,\Psi(x_i,x',\kappa_i,\lambda_i) \rangle$ pairs within a class $C_l$ of $N_l$ training points with minimal degradation in classification performance. 
Intuitively, if many points that characterize the class in question are sufficiently close to one another compared to points from negative classes, then we expect redundancy in their \PSI~responses. By thresholding on a minimum redundancy probability, we can select a subset of points that characterize the class.
While many strategies can be used for this selection, we wish to select the minimum number of points required to cover the class. 
We can formulate this strategy as a minimization problem: Let $x_i$ be a point in the class of interest and $\Psi(x_i,x',\kappa_i, \lambda_i)$ be its corresponding model. Without loss of generality, let $x_j$ be another point in the same class with model $\Psi(x_j,x',\kappa_j, \lambda_j)$. Let $\varsigma$ be the probability threshold above which to designate redundancy of the pair $\langle x_j,\Psi(x_j,x',\kappa_j, \lambda_j) \rangle$ with respect to $\langle x_i,\Psi(x_i,x',\kappa_i, \lambda_i) \rangle$, such that if $\Psi_i(x_i,x_j,\kappa_i,\lambda_i) \geq \varsigma$ then $\langle x_j,\Psi(x_j,x',\kappa_j, \lambda_j) \rangle$ is redundant with respect to $\langle x_i,\Psi(x_i,x',\kappa_i, \lambda_i) \rangle$. Finally, let $I(\cdot)$ be an indicator function such that 
\begin{alignat}{2}
 \begin{cases} 
   I(x_i) = 1 \text{ if } \langle x_i,\Psi(x_i,x',\kappa_i,\lambda_i) \rangle \text{ kept}\\
   I(x_i) = 0 \text{ otherwise}.
    \end{cases}
\end{alignat} 
If $x_i$ and $\Psi(x_i,x',\kappa_i,\lambda_i)$ are retained, they become \textit{extreme vectors} defining the final model. We can then express our optimization strategy in terms of this objective function:
\begin{alignat}{2}
  & \text{minimize } \sum_{i=1}^{N_l} I(x_i)\label{eqn:minimization} \text{ subject to}\\
  & \forall j \exists i | I(x_i) \Psi(x_i,x_j,\kappa_i,\lambda_i) \geq \varsigma.\label{eq:constraint}
\end{alignat}

The constraint (Eq.~\ref{eq:constraint}) requires that every $\langle x_i, \Psi(x_i,x', \kappa_i,  \lambda_i) \rangle$ pair be an EV or be covered by at least one other pair. Note that the implicit binary constraint in the range of $I(\cdot)$ makes the optimization an integer linear programming problem. The formulation in Eqs.~\ref{eqn:minimization} and~\ref{eq:constraint} is a special case of Karp's Set Cover problem. We can see this by defining a coverage set of indices $s_i \equiv \{ j \in \{1,..,N_l\} |  \Psi_i(x_i,x_j, \kappa_i, \lambda_i) \geq \varsigma\}$ for each $\langle x_i, \Psi_i(x_i,x', \kappa_i, \lambda_i) \rangle$ pair and a universe $U=\{1,..,N_l\}$. The objective of the Set Cover problem is then to select the minimum number of sets that contains all elements of $U$. 
While Set Cover is NP-hard, we employ the greedy approximation described in \cite{slavik1996tight} that offers a polynomial time $(1+\ln(N_l))$ approximate solution (\textit{cf}. Theorem 2 in~\cite{slavik1996tight}). This algorithm offers the smallest error bound for any polynomial time approximation. The greedy approximation entails selecting the sets of highest cardinality at each step. The upper bound in approximation error is $(1-o(1))\cdot\mathrm{ln}(N_l)$, where $N_l$ is the cardinality of the universe of set elements.

 Note that with the model fitting, an outlier is generally covered by a point from another class (see Fig.~\ref{fig:teaser}), and such outlier points are also unlikely to cover many other points. Thus, outliers are added to the coverage set very late, if at all. This is not an \textit{ad hoc} assumption; it is an outcome of the process of minimizing the number of points that cover all examples. Like the inherent softness of the margin, this is an inherent part of the model-reduction approach that follows from the EVT-modeling. 

\subsection{Incremental Learning}

Once EVs have been acquired for the EVM, the model can be updated with a new batch of data by fitting \PSI-models for the new data using both the current EVs and all points in the new batch of data. The new EVs are obtained by running model reduction over both the old EVs and the new training points. While new points can be added individually, adding data in batches will result in more meaningful fits because batches represent a richer distributional sample at each increment.
This means that newly added training points may or may not become EVs and new classes can also impact previously learned models and EVs. 
The efficient model reduction technique discussed in Sec.~\ref{sec:model_reduction} allows the EVM to limit the size of its models either probabilistically via an explicit selection of $\varsigma$ or by a specific maximum number of EVs in a max-$k$ cover greedy approach~\cite{mirzasoleiman2013distributed}. 
This allows the EVM to scale to many different incremental problems via different modes of operation. 
In Sec.~\ref{sec:imagenet}, for example, we choose a static $\varsigma$ and perform model reduction using this threshold at each training increment for classes to which data get added.
While this limits the growth in model size, the number of EVs still increases over time. 
Bounded optimization would specify a maximum per-class size or maximum total size, recalculating $\varsigma$ at each increment. 
Alternatively, maximum model sizes could be pre-specified with model reduction performed only when the maximum size is violated. 
Thus, the EVM is not only an incremental classifier, but it is an incremental classifier whose size can be controlled at each learning increment.



%
%
%

\section{Experimental Evaluation}
\label{sec:experiments}

We conducted several evaluations in which we compared the EVM to other open set and open world classifiers on published benchmarks, including the state-of-the-art for both problems. 
To ensure valid comparisons and consistency with previous research where applicable, we report the results of these evaluations using the same evaluation measures and thresholds that were used in the original benchmark settings. 
For all of the experiments in this section we selected a $\varsigma$ value of 0.5 based on the probabilistic assumption that points with greater than 50\% probability of being covered by others in that class are redundant with respect to the model.

\subsection{Multi-class Open Set Recognition on Letter}

\begin{figure}[t]
\begin{center}
  \includegraphics[width=\linewidth]{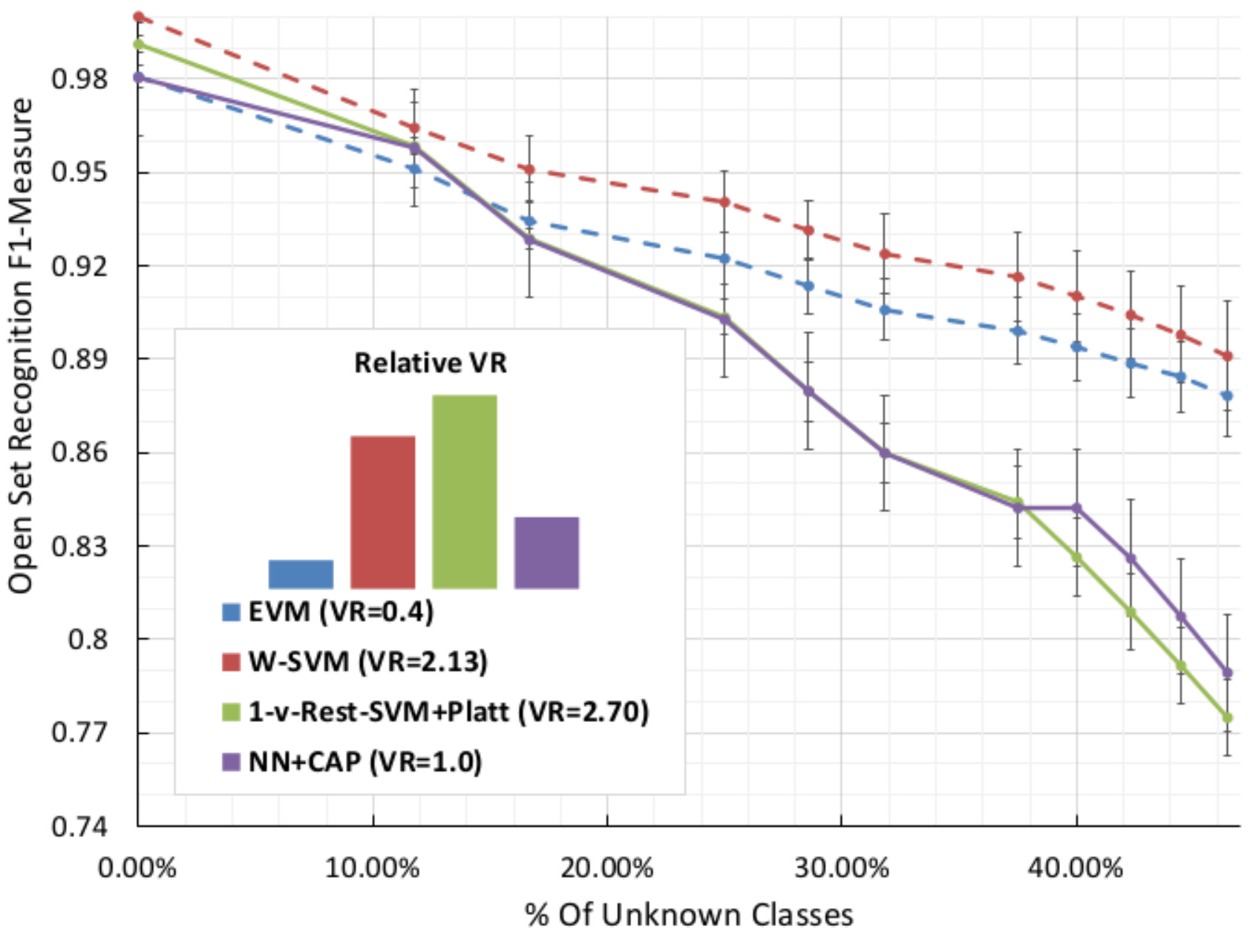}\vspace*{-4ex}
\end{center}
\caption{\small \textit{Multi-class open set recognition} performance on OLETTER. The $x$-axis represents the percentage of classes in the test set that were unseen during training. Error bars show standard deviation. The EVM is comparable to the existing state-of-the-art W-SVM~\cite{walter2014} in F1-Measure, but at a substantial savings in training efficiency and scalability, as reflected in the vector ratio (VR). The EVM's VR \textit{is an order of magnitude smaller} than the two SVM-based models.
Both EVM and W-SVM algorithms have favorable performance degradation characteristics as the problem becomes more open. The two other  probabilistically calibrated algorithms degrade far more rapidly. Hyperparameters of $\tau=75$ (tail size) and $k=4$ (number of EVs to average over) were used in this evaluation, and were selected using the same training set cross validation technique as the W-SVM.}
\label{fig:openset-letter}
\vspace{-2em}
\end{figure}

\begin{figure*}[!t]\begin{center}\includegraphics[width=0.5\linewidth]{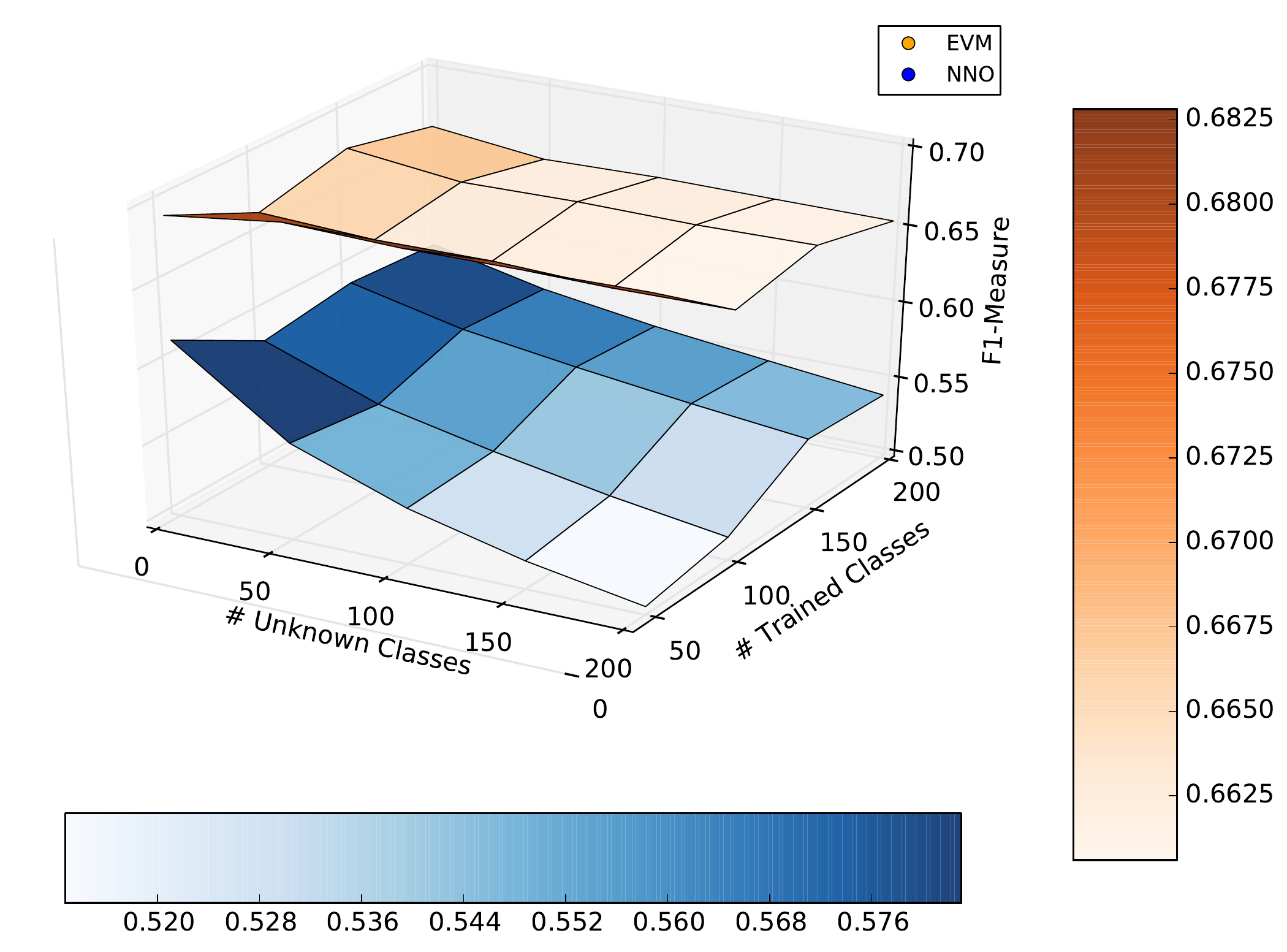}\includegraphics[width=0.5\linewidth]{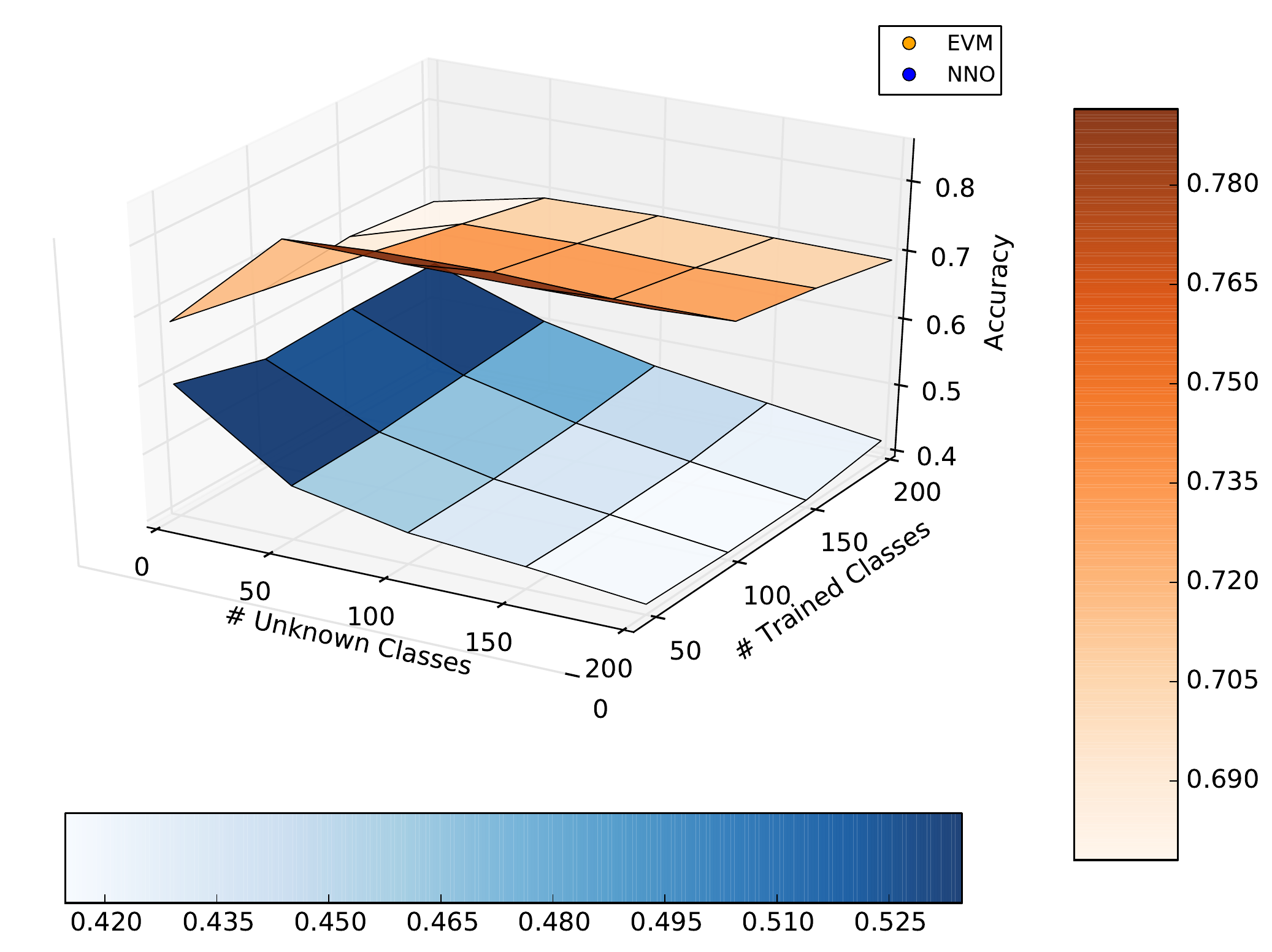}\end{center} \vspace*{-3ex}
  \caption{\small  Open world performance of the EVM and NNO algorithms on the open world ImageNet benchmark in terms of both F1-mearsure (left) and accuracy (right).
Initial training considering 50 classes was performed, then classes were incrementally added in groups of 50. 
The EVM dramatically outperforms NNO at all points evaluated and the divergence of the respective surfaces suggests superior scalability with respect to both the number of additional training classes and the number of additional unknown classes. Hyperparameter values of $k=6$ and $\tau=33998$ were selected using the cross validation procedure discussed in the text.}
\label{fig:openworld_imagenet}
\vspace{-1em}
\end{figure*}

To establish the ability of the the EVM to identify unknown classes, we first looked at the problem of multi-class open set recognition. Scheirer et al. developed the OLETTER protocol~\cite{walter2014} to evaluate classifiers in an open set recognition context.
It consists of randomly selecting 15 distinct labels from the Letter dataset~\cite{frey1991letter} as known classes during training and adding unknown classes by incrementally including subsets of the remaining 11 labels during testing. This process is repeated over 20 folds to calculate averages and error bars. For consistency with~\cite{walter2014}, we report results in terms of F1-Measure (over precision and recall), and dynamically set a threshold on open space
$\delta = \frac{1}{2}\left(1 - \sqrt{\left(2 \times C_T/(C_R+C_E)\right)}\right)$, where $C_R$ is the number of classes to be recognized (common to training and testing), $C_T$ is the number of classes used in training, and $C_E$ is the number of classes used in evaluation (testing).

For the OLETTER evaluation, we used Euclidean distance to compute the margins for our EVM.
We evaluate model compactness in terms of the \textit{vector ratio},  defined as $\hbox{VR} = \frac{\text{\# points retained by model}}{\text{ total \# training points}}$. The $\hbox{VR}$ is a scaled form of the support vector ratio introduced by Vapnik to provide an approximation of generalization error~\cite{vapnik1998statistical,valentini2004bias}. 
This allows us to compare the scalability of different nonlinear models.

Fig.~\ref{fig:openset-letter} shows results for all of the evaluated algorithms, including the open set-specific W-SVM~\cite{walter2014}, which is currently the best performing algorithm in the literature for this problem. Results for Nearest Neighbor (NN) classifiers with CAP probability estimation~\cite{walter2014}, and 1-vs-Rest RBF SVMs with Platt's probability estimator~\cite{Platt1999} are also shown. Other calibrated models assessed in~\cite{walter2014} performed significantly worse and are not shown. We selected RBF parameters for the SVMs via 5-fold cross validation on the training dataset, using a grid of $C=2^{-5},2^{-3},\ldots ,2^{15}$ and $\gamma=2^{-15},2^{-13},\ldots ,2^{3}$, consistent with~\cite{hsu2003practical}. 

The EVM performs comparably to the W-SVM, and outperforms all other algorithms.  The W-SVM is certainly a viable algorithm for this dataset, but its slight advantage comes at a greater cost than the EVM, requiring \textit{two} trained SVM models (per-class: one 1-class and one binary) for its operation.  The vector ratio for this experiment is computed for models trained on all 26 classes. For the evaluation in Fig.~\ref{fig:openset-letter}, the EVM's vector ratio is an order of magnitude smaller than that of any of the SVM models, indicating that for the chosen $\varsigma$ (0.5), fewer than half of the training data points were included as EVs. The number of support vectors in the SVM models is greater than the number of points in the training set. This is due to redundancy in support vectors kept in the multi-class regime. Although Platt-Calibrated SVM model processing and storage costs can be reduced by caching duplicated support vectors, the computational savings is less feasible for the W-SVM, since it uses different SVM models with multiple kernels. This is because different RBF kernels require different calculations even if they are centered on the same point. Also, for the EVM, and unlike the W-SVM, we can easily obtain a lower vector ratio while minimizing any degradation in accuracy. We analyze this tradeoff in Sec.~\ref{sec:discussion}. Finally, we would like to mention that, apart from the EVM, none of the classifiers whose performance is depicted in Fig.~\ref{fig:openset-letter} support incremental learning, so they cannot be applied to open world problems.

\subsection{Open World Recognition on ImageNet}
\label{sec:imagenet}

Open world recognition, the problem we are most concerned with, consists of three distinct phases~\cite{bendale2014towards}: one initial training phase, followed by alternating phases of open set recognition and updates to incorporate newly labeled data into the model.
We evaluated open world recognition performance on a benchmark protocol that uses an open world partition of ImageNet introduced in \cite{bendale2014towards}.
\comment{
\tb{Realy.. oooks liek same frequency +50, to me} One small change that we introduced was an increase in frequency of evaluation with respect to number of unknown classes added to the test set. 
We introduced this change to better observe the behavior of each classifier as the problem increased in openness.
}
The protocol consists of an initial training phase on 50 classes, performing classification with 0, 50, 100, 150, and 200 unknown classes in the test set.  
At each increment another group of 50 classes is added, and classification is again performed on test instances from known classes, with samples from 0, 50, 100, 150, and 200 additional unknown classes. 
One manner in which we depart from the protocol in~\cite{bendale2014towards} is that instead of using dense SIFT features, which are no longer state-of-the-art, we use a 4,096-dimensional deep feature space representation derived from the fc7 layer of AlexNet~\cite{krizhevsky2012imagenet}. This is a much better representation for today's visual recognition tasks.

We compared the EVM's computational performance with the state-of-the-art NNO algorithm~\cite{bendale2014towards} on this benchmark, as well as the incremental KDE classification algorithm employed by Haines and Xiang in \cite{haines2014active} that had heretofore not been tested on this benchmark, or any particularly large-scale benchmark, in terms of either number of samples or dimensionality. While both the EVM and the NNO algorithms completed in a matter of hours, the incremental KDE classifier had enrolled fewer than 40 samples after 24 hours. Assuming a constant rate of enrollment, it would take approximately 18 years for the benchmark training experiment to complete, leading us to conclude that the approach, at least as implemented, is not scalable, and therefore report recognition results for just the EVM and NNO algorithms.

For training and testing, we used the ILSVRC2014 training and validation partitions respectively. 
For consistency with~\cite{bendale2014towards}, we selected an NNO rejection threshold via 3-fold cross-class validation on the first 50 classes from the training set. 
This is also how we selected $\delta$ for the EVM, searching over a range of $[ 0.05, 0.1, \hdots, 0.3 ]$. Hyperparameters $k$ and $\tau$ were obtained via a Bayesian search using the hyperopt library~\cite{bergstra2013hyperopt}.
The Bayesian search conducted with hyperopt consisted of 50 iterations of three-fold cross-validation over the first 50 classes of the training set. 
The hyperparameter ranges consisted of 1-10 for $k$ and 100-32,000 for $\tau$.
Noting that 3-fold cross validation reduces the size of the training set by 1/3, and assuming a rough proportionality on these hyperparameters to training set size (and increment batch size), we multiplied the selected values by 1.5 and rounded to the nearest integer value to arrive at hyperparameter selections used for training. 

During our ImageNet experiments, we found that Euclidean distance led to poor performance when computing margins for the EVM. 
This is consistent with the previous finding that Euclidean distance does not work well when comparing deep features of individual samples~\cite{aggarwal2001surprising}.
We therefore turned to cosine similarity, which is a commonly used measure of divergence between samples in a deep feature space. 
NNO still performed reasonably well when using Euclidean distance; which we suspect is because it compares samples with respect to class means.
However, when cosine similarity was used, the classifier rejected nearly all test samples as unknown. 
Thus, we report the best results for each classifier, using cosine similarity for the EVM and Euclidean distance for NNO. 
These results are shown in Fig.~\ref{fig:openworld_imagenet}.

The EVM consistently and dramatically outperforms NNO, both in terms of accuracy and F1-Measure.
Readers might wonder why the EVM's accuracy increases as the number of classes are added; this is because the EVM is a good rejector and is able to tightly bound class hypotheses by their support. However, it does so while simultaneously maintaining reasonable classification performance. While NNO models the deviation from class means according to a support bound, such a rigid and over-simplified model of each class yields a poorer performance trend than the EVM delivers, hence the overall divergence of the surface plots in Fig.\ref{fig:openworld_imagenet}. Finally, we note that with the chosen value $\varsigma=0.5$, the EVM provides significant reductions in model size over using all points (cf. Table \ref{tab1}).

\begin{table}[t]
\begin{center}
\caption{Numbers of extreme vectors, cumulative numbers of training points used, and vector ratios after each batch is added.}
\label{tab1}
\begin{tabular}{ | l | r | r | r | r |}
\hline
Batch & 0-50 & 50-100 & 100-150 & 150-200 \\ \hline
\# EVs & 16309 & 35213 & 52489  & 74845 \\ \hline
\# Points & 64817 & 129395 & 194217 & 255224 \\ \hline
Vector Ratio & 0.25 & 0.27 & 0.27 & 0.29\\ \hline
\end{tabular}
\end{center}
\vspace{-2em}
\end{table}

\section{Practical Considerations for the EVM}
\label{sec:discussion}





The model reduction strategy that we employed in Sec.~\ref{sec:experiments} of selecting a threshold $\varsigma$ and running the Set Cover algorithm is a simple way to increase efficiency at classification time and achieve a compact representation of the training set.
While our experiments in Sec.~\ref{sec:experiments} used a probabilistically motivated choice of $\varsigma=0.5$, what constitutes a good redundancy threshold in practice is often tied to the computational, storage, or time budget of the problem at hand. The model must therefore be sized accordingly so that it does not exceed the maximum budget, while still maintaining as good performance as possible.
We refer to this as the \emph{budgeted optimization} problem in which the objective is to obtain the most representative model achievable that meets but does not exceed the budget. 
We can perform this selection via a binary search for $\varsigma$, for which the optimization, given a target budget, most closely returns the requested number of EVs. 
Since the greedy optimization selects EVs in order of their coverage, we can easily retain only the most important of these EVs. 
This allows EVM classifiers to be approximately portable across many device types of heterogeneous computational capabilities.

We performed an evaluation of this technique on the closed set Letter dataset, using all points in the training set ($\varsigma=1.0$), which yielded a base accuracy of 96\%. 
Using 50\% of the data ($\varsigma = 0.492$) or 40\% of the data ($\varsigma = 0.186$) yielded accuracies comparable  to using all points.
Reducing to 10\% of the training points ($\varsigma = 0.008$) yielded an accuracy of 92\%.
This suggests that budgeted optimization is quite effective for classifier compression/portability, and that $\varsigma$ can assume a very wide range with minimal impact on classification performance.

With respect to computational efficiency, much of the EVM's training procedure can be performed independently on a class-by-class basis, making the algorithm well suited for a cluster or GPU implementation. 
Each statistical fit, made via Maximum Likelihood Estimation (MLE), is fast and constant in time complexity, due to a cap on the maximum number of iterations. Model reduction is $O(N_l^2)$, where $N_l$ is the number of points in class $C_l$. The complexity of each tail retrieval, given $N$ training points, can be reduced from $O(NlogN)$ to $O(\tau logN)$ by introducing space partitioning data structures, \textit{e.g.}, $k$-$d$ trees~\cite{bentley1975multidimensional}. However, employing them may impose constraints on the types of measures used to compute the margin --- the common requirement that the distance function be a strict metric on the hypothesis space precludes quasi-metrics such as cosine similarity.


While the EVM may appear to be highly parameterized, {\bf the per point Weibull scale and shape parameters are purely data derived and are automatically learnt during training}. They are a function of the MLE optimization process, which we must distinguish from hyperparameters selected prior to fitting.
The only hyperparameters are $\tailsize$, $k$, and $\varsigma$. 
As we have previously discussed, barring hard model size constraints (\textit{e.g.}, in budgeted optimization), a wide range of cover probability $\varsigma$ values work well, including the probabilistically driven choice of 0.5 used in the experiments in Sec.~\ref{sec:experiments}.
With respect to the value $k$, the number of models to average,  we found that searching over $\{1,\hdots,10\}$ on a validation set yielded slightly better results on the test set (with further decreased performance for $k > 10$), although performance variations over this parameter range typically accounted for less than $2\,\%$ in accuracy or $F1$-measure respectively. 
Thus, while $\varsigma$ and $k$ might have a slight impact on performance, the vast majority of performance variation can be attributed to only a single hyperparameter, namely $\tailsize$.

Unfortunately, EVT provides no principled means of selecting the tail size $\tailsize$. The theory only dictates the family of distribution functions that will apply and proves convergence for an unspecified $\tailsize$. 
We use cross-class validation for  selecting $\tailsize$ that yields optimal cross-validation accuracy with missing classes.  For the selection of $\tailsize$ obtained on the OLETTER dataset through cross-class validation, the value $\tailsize=75$ accounts for only 0.5\% of the 15K OLETTER training set and that small a fraction represents plausible extrema.  However, 
the results of applying the cross-class validation methodology discussed in Sec.~\ref{sec:imagenet} for our ImageNet experiments might raise questions since the selected tail size ($\tau = 33998$) consists of approximately half the number of samples in the training set.  This result is surprising and counter-intuitive because one would not ordinarily think of ``half of the data'' as being extreme values.   But looking at a fraction of data is probably the wrong way to determine the boundary or extreme points.  While the EVM's \PSI-models are fit on 1D margins, the high dimensionality of the feature space (4,096) translates into many more directions yielding ``boundary points'' than for the relatively low-dimensional (16) OLETTER dataset.  Normalizing for feature dimensionality and the number of classes we find both examples are similar. Dividing the tail size used for the first batch of ImageNet by the dimensionality of the feature space and number of classes yields 0.17 points per dimension per class. Doing the same for OLETTER, we obtain 4.69 points per dimension and an average of 0.17 points per dimension per class.


\section{Conclusion}
\label{sec:conclusion}

Perhaps the most important conclusion of this work is that the EVM is able to do kernel-free nonlinear classification. Interestingly, the EVM shares some relationships with radial basis functions. When $\kappa=2$, the functional form of Eq.~\ref{eq:weibull_cdf} is the same as a Gaussian RBF, and when $\kappa=1$ it is the same as an exponential or Laplacian RBF. While these $\kappa$ values can occur in practice, $\kappa$ assumes a much broader range of values, which are generally larger. Furthermore for $\kappa>2$, Eq.~\ref{eq:weibull_cdf} is not a Mercer kernel.
Alternatively, if one approximates Eq.~\ref{eq:weibull_cdf} by a weighted sum of Gaussians (or Laplacians) we have two different ways of viewing a Gaussian (or Laplacian) RBF kernel as an approximation of a \PSI-model. While the \PSI-model parameters vary in scale and shape with the bandwidth and density of the data set, in a Gaussian approximation the number of kernel elements and/or the accuracy of approximation must vary spatially. The EVM requires the fewest points for the margin distribution and its \PSI-model. 
{\bf For the EVM, we do not make an {\em ad hoc} assumption of a kernel
 trick nor a {\em post hoc} assumption of a particular kernel function; the
 functional form of the \PSI-model is a direct result of EVT being used
 to model input space distance distributions.} 
 
 The Weibull fitting ensures that a
small number of mislabeled points or other outliers will not cause the
margin estimated from the Weibull to be at that location. If
the fitting includes more distant points, the \PSI-model will broaden in
scale / shape providing a naturally derived theory for the ``softness'' in its
margin definition. However, the overall optimization with Set Cover currently
lacks a parameter to adjust the risk tradeoff between positive and negative classes.
Future directions of research may include directly extending the EVM by obtaining a better parameterized soft-margin during Set Cover, perhaps by adding weights to balance soft-margin errors and formulating the problem in terms of linear programming. Another potential extension would be to incorporate 
margin weights in a loss function in an SVM-style optimization algorithm.

\vspace{-1em}
\section*{Acknowledgement}
This work was supported by the National Science Foundation, NSF grant number IIS-1320956: Open Vision -- Tools for Open Set Computer Vision and Learning.
\vspace{-1em}
\bibliographystyle{IEEEtran}
\bibliography{paper}



\section*{Supplementary material}

# The Extreme Value Machine – Supplemental Material

## May 19, 2017

## 1 Algorithm Details

In this section we provide the details of the algorithms presented in the main text of the paper. The process of training the EVM is provided in Alg. 1.

**Algorithm 1** EVM Training

1: **function** TRAIN EVM($X,y,\tau,\varsigma$)    ▷ $X$ and $y$ are data and labels for all classes $\mathcal{C}$.
2:   **for** $\mathcal{C}_l \in \mathcal{C}$ **do**
3:     $\Psi_l \leftarrow$ Fit $\Psi(X,y,\tau,C_l)$    ▷ See Alg. 2.
4:     $I \leftarrow$ Reduce($X_l, \Psi_l, \varsigma$)    ▷ See Algs. 3 and 4.
5:     $\Psi_l \leftarrow \Psi_l[I]$
6:     $X_l \leftarrow X_l[I]$
7:     $y_l \leftarrow y_l[I]$
8:   **end for**
9: **end function**

Figure 1: Algorithm for EVM training. The function calls are detailed in Algs. 2, 3, and 4.

For brevity, we adopt array notation. The function takes four arguments (cf. line 1): $X$ and $y$ correspond to all training data and labels respectively, $\tau$ is the tailsize, and $\varsigma$ is the coverage threshold. The training function then iterates through every class $\mathcal{C}_l \in \mathcal{C}$ (line 2), fitting a vector of $\Psi$-models ($\Psi_l$) for that class (line 3). The fitting function is presented in Alg. 2 and takes four arguments: $X$, $y$, $\tau$, and the class identifier $C_l$. At each fit, only the $\Psi$-models for the points and labels corresponding to class $C_l$, *i.e.*, $X_l$ and $y_l$ are considered. The function returns $\Psi_l$ – a vector of Weibull parameters. Next, the model reduction routine is called (line 4), which takes as arguments $X_l$, $\Psi_l$, and coverage threshold $\varsigma$. This routine can correspond to Algs. 3 or 4. A vector, $I$, of indices is returned by this routine, and is then used to select Weibull parameters $\Psi_l$, points $X_l$, and labels $y_l$ to keep (lines 5-7). Note that for conceptual clarity, we have presented Alg. 1 in an itera-

tive fashion. However, it is possible to parallelize the loop in lines 2-8, fitting each class separately. The extent to which that is effective is an implementation-level consideration that depends on the hardware at hand, particularly since each $\Psi_{li}$ in the fitting algorithm (Alg. 2) can also be parallelized as can pair-wise distance computations within all subroutines (Algs. 2-4).

**Algorithm 2** $\Psi$-Model Fitting

1: **function** FIT $\Psi(X,y,\tau,C_l)$
2:   $D = \text{cdist}(X_l, X_{\neg l})$    ▷ $N_l \times N_{\neg l}$ dimensional matrix.
3:   **for** $i = 1$ **to** $N_l$ **do**
4:     $\Psi_{li} \leftarrow$ Weibull_fit_low($\frac{1}{2} \times$ sort($D_i$)$[: \tau]$)
5:   **end for**
6:   return $\Psi_l$
7: **end function**

Figure 2: Algorithm for fitting $\Psi$-Models. This is called from Alg. 1.

The fitting routine, depicted in Alg. 2, takes as arguments (cf. line 1) all sample points $X$ with labels $y$, tailsize $\tau$, and class identifier $C_l$. We use notation $X_l$ to refer to points belonging to class $C_l$ and $X_{\neg l}$ to refer to all other points. In line 2, the $N_l \times N_{\neg l}$ matrix of pairwise distances is computed, then each row of that matrix is sorted in ascending order and an MLE Weibull fit is performed on the nearest $\tau$ margins (line 4). Weibulls are then returned as $\Psi_l$ (line 6).

The first model reduction routine, which approximates the solution to Set Cover is shown in Alg. 3. Employing this algorithm yields a non-pre-determined number of extreme vectors. The routine takes three arguments (line 1): $\Psi_l$, $X_l$, and $\varsigma$ – the coverage threshold. First, in line 2, pairwise distances between points are computed to yield an $N_l \times N_l$ matrix of distances, then for each of these points, respective probabilities of association with

---

**Algorithm 3** Set Cover Model Reduction

---

1: **function** SET COVER($\Psi_l, X_l, \varsigma$)
2:     $\widetilde{D}$ =pdist($X_l$)   ▷ An $N_l \times N_l$ matrix.   ▷ Subsets of indices.
3:     **for** $i = 1$ **to** $N_l$ **do**
4:        **for** $j = 1$ **to** $N_l$ **do**
5:           **if** $(\Psi_{li}(\widetilde{D}_{ij}) \geq \varsigma)$ **then**
6:              $\mathbb{S}_i$.insert($j$)
7:           **end if**
8:        **end for**
9:     **end for**
10:    $\mathbb{U} = \{1, ..., N_l\}$   ▷ Universe.
11:    $\mathbf{C} \leftarrow \{\}$   ▷ Covered elements.
12:    $I \leftarrow []$   ▷ Cover indices.
13:    **while** ($\mathbf{C} \neq \mathbb{U}$) **do**
14:        $ind \leftarrow \text{argmax}_i (\mathbb{S}_i - \mathbf{C})$
15:        $\mathbf{C} \leftarrow \mathbf{C} \cup \mathbb{S}_{ind}$
16:        $I$.append($ind$)
17:        $\mathbb{S}$.remove($\mathbb{S}_{ind}$)
18:    **end while**
19:    **return** $I$
20: **end function**

---

Figure 3: A fast algorithm for model reduction which leverages an approximation to Set Cover. Using this yields a variable number of extreme vectors.

the point governing the corresponding row are computed (line 5). For each point-probability in the $i$th row, if it is above $\varsigma$, it is added to the respective set of covered points associated with the $i$th extreme vector, *i.e.*, $\mathbb{S}_i$. In lines 13-18, values $i$ of the greatest cardinality $\mathbb{S}_i$ are greedily selected without replacement until all points within class $\mathcal{C}_l$ have been covered, and the respective indices of these points are appended to vector $I$. Once all points in the class are covered $I$ corresponds to the indices of points, $\Psi$-models, and labels to retain as extreme values. This vector is returned to the calling routine in line 19.

An extension of the Set Cover model reduction algorithm in Alg. 3 is the fixed-size reduction algorithm in Alg. 4, in which a bisection is performed on $\varsigma$ to attain a fixed number of extreme vectors ($N_{max}$). The first two arguments to this algorithm ($\Psi_l$ and $X_l$) are the same as for Alg. 3. $N_{max}$ corresponds to the maximum number of extreme vectors, and $\delta_P$ corresponds to a tolerance on the difference between chosen $\varsigma$ values over iterations at which to terminate the bisection.

The bisection is carried out in lines 5-24. Upon

---

**Algorithm 4** Model Reduction

---

1: **function** FIXED-SIZE REDUCTION($\Psi_l, X_l, N_{max}, \delta_P$)
2:    $\varsigma_{min} \leftarrow 0.0$
3:    $\varsigma_{old} \leftarrow \varsigma_{max} \leftarrow 1.0$
4:    $I_{old} \leftarrow [1, .., N_l]$
5:    **while** true **do**   ▷ Infinite while loop.
6:        $\varsigma \leftarrow (\varsigma_{min} + \varsigma_{max})/2$
7:        $I \leftarrow \text{Set.Cover}(\Psi_l, X_l, \varsigma)$   ▷ Call to Alg. 3.
8:        $M \leftarrow \text{length}(I)$
9:        $C_1 \leftarrow (N - N_{max} \geq M - N_{max} > 0)$
10:       $C_2 \leftarrow (N - N_{max} \geq M - N_{max} < 0)$
11:       $C_3 \leftarrow (N - N_{max} < M - N_{max})$
12:       $C_4 \leftarrow ((M = N_{max}) \text{ or } abs(\varsigma - \varsigma_{old}) \leq \delta_P)$
13:       **if** $C_1$ **then**
14:          $\varsigma_{max} \leftarrow \varsigma$
15:       **else**
16:          **if** $C_2$ or $C_3$ **then**
17:             $\varsigma_{min} \leftarrow \varsigma$
18:          **end if**
19:       **end if**
20:       **if** $C_4$ **then**
21:          **return** $I[: N_{max}]$   ▷ First $N_{max}$ elements of $I$
22:       **end if**
23:       $I \leftarrow I_{old}; \varsigma_{old} \leftarrow \varsigma$
24:    **end while**
25: **end function**

---

Figure 4: Algorithm for fixed-size model reduction. The appropriate value of $\varsigma$ is arrived at to yield a model of a fixed size via a binary search on the output of the Set Cover algorithm. While a fixed number of extreme vectors can be returned, this algorithm takes longer to compute, as it must call Alg. 3 several times.

the completion of the search, the value of $\varsigma$ that most closely, within $\delta P$ yields $N_{max}$ extreme vectors, with at least $N_{max}$ extreme vectors, is selected and out of these extreme vectors, the $N_{max}$ of highest coverage are returned.

Since Alg. 4 allows a principled mechanism for the EVM to use a fixed subset of Extreme Vectors, it is technically possible, by extension, to have Vector Ratio (discussed in the main text) as a hyperparameter of the EVM rather than $\varsigma$. However, this does come at an increased computational cost, since multiple calls to Alg. 3 have to be performed. Examining the ramifications of using Vector Ratio as a hyperparameter is a topic which we leave for future work.

## 2 Parameter Distributions

While the EVM *significantly* outperforms NNO in terms of both accuracy and F1-measure (see Sec. V. of the main text), and is easily parallelized as the Weibull fits do not depend on one another, computing class means is still noticeably more efficient than computing maximum likelihood estimates of $\lambda$ and $\kappa$ over all extreme vectors. For very large datasets or for training with limited resources, it is natural to ask: can we make the EVM more efficient by choosing principled parameter values and avoiding fits altogether, or using better initializations on parameter values to speed up MLE convergence?

While formulating such a classifier is largely beyond the scope of the main text, as a preliminary exploration of this question, we illustrate the distribution of shape and scale parameters in Fig. 5 for both of the datasets that we evaluate in the main text. For both ImageNet and Letter, scale parameters were very close to one. This is interesting because the datasets are quite different, both in terms of characteristics and in terms of dimensionality. Moreover, different metrics were chosen over which to compute the margins. However, the distributions behave quite similarly, with a slightly longer tail for the Letter dataset. The approximately equal mean scale of $\Psi$-models between datasets is not surprising, due to the normalization of input data during pre-processing. What is more striking is the similar behavior of the distributions.

Shape parameter distributions assume a much wider range for both ImageNet and Letter and the distributions look little alike between the two datasets. Thus, it is not immediately clear what would constitute a principled choice for $\kappa$. In order to ascertain the dependence of the model on $\kappa$, we selected values of $\kappa = 1$ and $\kappa = 2$, and attempted to run the open set protocol on OLETTER. However, this resulted in very poor open set recognition performance.

## 3 EVM-SVM Analogy

The Extreme Value Machine has some components that are analogous to those of Support Vector Machines, but the respective models and components behave in fundamentally different manners. While $\Psi$-models are kernel-like radial basis functions, they do not obey Mercer's Theorem and are therefore not Mercer kernels for values of $\kappa > 2$. As can be seen in the figures from the previous

section, $\kappa > 2$ occurs quite frequently, so the EVM cannot be interpreted as a kernel machine.

Further, the EVM optimization objective of selecting the smallest number of extreme vectors to cover all points within the class within some probability differs from the SVM objective of selecting support vectors that yield maximum soft margin. SVM support vectors carry little information about class inclusion/structure over all space but considerable information about structure of the decision boundary, while EVMs carry information with respect to probability of inclusion over each entire class.

While one can draw an analogy between $\varsigma$ and the $C$ parameter in an SVM, it is a very loose analogy – these hyperparameters serve a fundamentally different purpose: $C$ controls the softness of the margin for an SVM, whereas $\varsigma$ controls the number of extreme vectors. The loose analogy is that $C$ and $\varsigma$ control what support vectors or extreme vectors constitute the model, but they do so in much different ways. The primary objective of $C$ is to tune the softness of the margin, whereas the primary objective of $\varsigma$ is to tune the size of the model. Certain choices of $\varsigma$ may have regularizing effects too, but assessing them is a topic that we leave to future work. A better analogy to the SVM's $C$ parameter might be $\tau$, as it does control the softness of the margin, but again, what constitutes margin is fundamentally different – the EVM's soft margin is a result of the EVM's point-wise margin distribution. Note also that modelsize is dependent on both $\tau$ and $\varsigma$, because larger values of $\tau$ result in more coverage and thus the same value of $\varsigma$ will yield fewer extreme vectors.

With respect to the W-SVM (referenced in the main text of the paper), it is also a fundamentally different model from the EVM: It consists a one-class SVM model and a multi-class SVM model, but the EVT fitting is done in a post-hoc manner, at the score level, and after the SVMs have been trained. The EVM, by contrast, performs EVT fitting directly during the optimization process, computed in feature space – not at the score level.

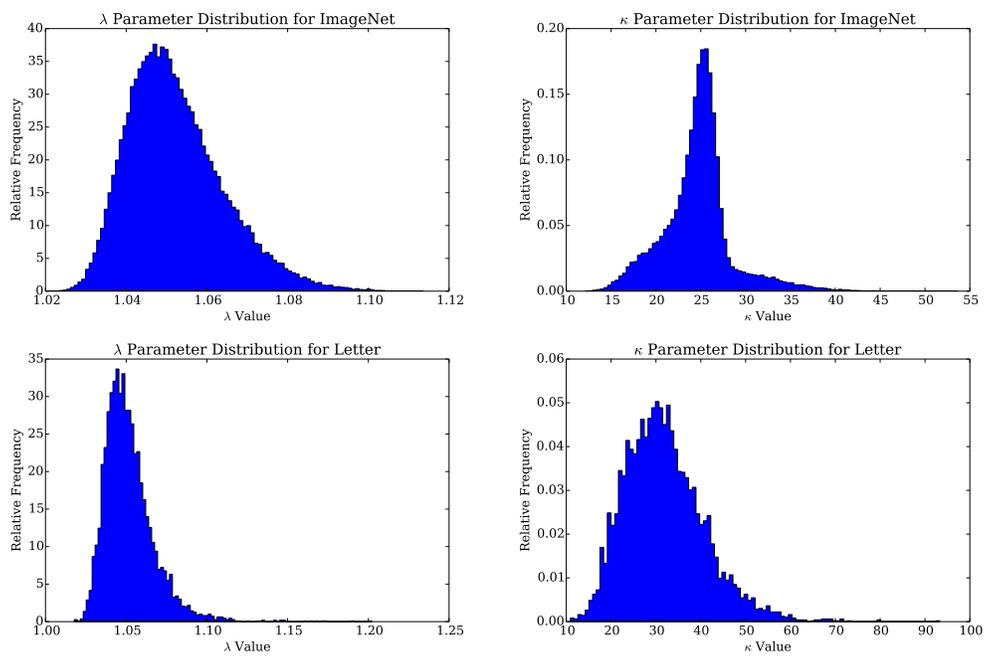

Figure 5: Distributions of $\Psi$-model scale ($\lambda$) and shape ($\kappa$) parameters over the first 200 classes of the ImageNet training set and over all classes of the Letter training set.



\end{document}